\documentclass[11pt,a4paper]{article}

\usepackage[margin=2.7cm]{geometry}
\usepackage{amsmath,amssymb,amsthm}
\usepackage{bm}
\usepackage{xcolor}
\usepackage[colorlinks=true,linkcolor=blue!60!black,citecolor=blue!60!black,urlcolor=blue!60!black]{hyperref}
\usepackage{algorithm}
\usepackage{algpseudocode}
\usepackage{enumitem}
\usepackage{soul}

\newtheorem{property}{Property}

\newtheorem{remark}{Remark}
\newtheorem{observation}{Observation}

\newcommand{\R}{\mathbb{R}}
\newcommand{\bP}{\mathbf{P}}
\newcommand{\bp}{\mathbf{p}}
\newcommand{\bK}{\mathbf{K}}
\newcommand{\bR}{\mathbf{R}}
\newcommand{\bT}{\mathbf{T}}
\newcommand{\bw}{\mathbf{w}}

\title{\textbf{Efficient and Robust Camera-independent Multiview 3D Geometric
Reconstruction from Noisy Monocular Depth Estimation and Multiple Point Matching}}
\author{Marius Leordeanu \\ \\
Institute of Mathematics of the Romanian Academy, Calea Grivitei 21, Bucharest \\
POLITEHNICA Bucharest, Splaiul Independenței 313, Bucharest \\
Contact: leordeanu@gmail.com}
\date{Last modified: September 3, 2026}

\begin{document}
\maketitle

\begin{abstract}

We present an efficient and robust method for 3D geometric reconstruction that is based solely on the camera-independent linear relationships among a given set of points, which are stable over time and robustly estimated using multiple point matches. We essentially learn, from correspondences between points across several frames, a linear geometric auto-regression matrix $\mathbf{W}$, which establishes how a point in 3D can be expressed as a linear combination of all the others. This matrix is constant and does not depend on the world coordinate system or the camera pose---it is an intrinsic property of the point set. We also show that the principal eigenvectors of $\mathbf{W}$, which all have eigenvalue $1$, provide a homogeneous representation of the 3D point configuration.

\textbf{The first version of our method takes advantage of noisy monocular depth maps} in order to obtain, from multiple frames, a robust geometric auto-regression matrix $\mathbf{W}$ of linear relationships between the 3D points. Thus, we build on recent advances in deep learning, which now provide monocular depth estimation models that are fast but very often noisy. Our approach handles noise through robust linear estimation over several frames (Section~\ref{sec:Spec_3D_1}).

\textbf{The second version of our method does not need monocular depth estimation maps}. It applies in cases of weak-perspective projection, when the linear combinations between the 3D points can be robustly estimated from their 2D projections in the image (Section~\ref{sec:Spec_3D_2}).

\textbf{Note that the camera projection matrix is never used in our derivations}. Consequently, our method does not recover camera pose, but only 3D structure. This is a key difference between our method and the related literature on 3D geometric reconstruction.

Throughout the paper, we gradually evolve the solutions towards their more general versions. We start by introducing a robust method for estimating and tracking the full 3D pose---3D location \emph{and} 3D rotation---of any point (or rigidly attached object, e.g.\ a virtual avatar) in a video, within a fixed world coordinate system, using only off-the-shelf components: a standard 2D point tracker, a monocular depth estimator, and a semantic segmentation network. The key idea is to track a set of $K$ \emph{anchor points}, lift them to 3D using (noisy) monocular depth, and represent the point of interest as a \emph{fixed} linear combination of the anchors. We prove that, when points are written in homogeneous coordinates, such a linear combination is invariant under any rigid transformation---provided the weights sum to one---so the weights can be estimated once and reused in every frame. Robustness to noisy depth comes from using many anchors ($K \gg 4$) and estimating the weights over multiple frames by least squares; the rotation component is recovered quickly using the same approach or by classical SVD alignment of the anchor sets. Semantic segmentation selects anchors in regions that are static with respect to the world. The method naturally handles occlusion of the tracked point and degrades gracefully when anchors are lost.

Next, we extend the idea to its more general form, for full 3D reconstruction of all anchors based on \emph{geometric auto-regression}: writing every anchor as a masked barycentric combination of the others yields a fixed-point equation $\mathbf{P} = \mathbf{W}\mathbf{P}$---a \emph{linear geometric auto-encoder}---whose eigenvalue-one eigenvectors are precisely the anchor homogeneous coordinates: the explicit 3D model within a common reference frame, up to an affine ambiguity.

This enables auto-correction of noisy 3D points by power iteration and, \textbf{remarkably, even 3D structure recovery from 2D near-orthographic projections alone, without any monocular depth map input}, by direct eigen-decomposition of a weight matrix estimated only from 2D point matching. 

This differs from classic structure from motion: it does not need to estimate camera motion, and it treats points not through projections that depend on camera pose, but as linear combinations of potentially many others---which could improve robustness while eliminating camera pose estimation altogether.

\medskip
\noindent \textbf{Keywords:} multiview 3D geometry, 3D pose tracking, geometric auto-encoder, spectral 3D geometry, avatar placement, temporally consistent 3D from monocular depth, 3D structure reconstruction from multiple views, affine combinations of 3D points.
\end{abstract}

\section{Initial Motivation}

We will start with the specific application of placing and tracking virtual content in real video---for instance, anchoring an avatar to a scene so that it stays put and rotates consistently as the camera moves. The task requires the full 3D pose (location and orientation) of a target point over time, expressed in a fixed world coordinate system.

Modern building blocks make an appealingly simple pipeline possible. Point trackers such as TAPIR~\cite{doersch2023tapir} and CoTracker~\cite{karaev2024cotracker} provide reliable long-term 2D tracks, and monocular depth foundation models such as Depth Anything~V2~\cite{yang2024depthanything} provide dense depth from a single image. However, three difficulties stand in the way of using them directly:

\begin{enumerate}[label=(\alph*)]
  \item monocular depth is \emph{noisy}, so a single lifted 3D point is unreliable;
  \item a 2D tracker follows only the point itself, and fails when the point is occluded or lies in a textureless region (zero image gradients, constant color);
  \item a 2D track provides \emph{position only}---it carries no rotation information, which is indispensable for correctly orienting an avatar with respect to the world (Remark~\ref{rem:avatar}).
\end{enumerate}

\noindent\textbf{Main idea.} We use standard tracking of $K$ anchor points, for which we have (noisy) monocular depth estimates, to accurately and robustly estimate both the 3D location and the 3D rotation (i.e.\ the 3D pose) of a target point within a fixed world coordinate system. The target is expressed as a linear combination of the anchors whose weights are \emph{constant across frames}; this constancy is not an approximation but an exact geometric property (Section~\ref{sec:property}), closely related to the barycentric control-point parameterization of EPnP~\cite{lepetit2009epnp}. Redundancy ($K \gg 4$ anchors, $N$ frames) turns the noisy per-point depth into a robust overconstrained least-squares problem, and semantic segmentation~\cite{kirillov2023sam} chooses anchors on structures that are fixed with respect to the world, in the spirit of semantic masking in dynamic SLAM~\cite{bojko2022selfimproving}.

\medskip
\noindent\textbf{Novelty.} A rapidly growing body of work tackles 3D point tracking directly with learned, feedforward models: SpatialTracker~\cite{xiao2024spatialtracker} and DELTA~\cite{ngo2025delta} lift 2D tracking into a learned 3D feature space, TAPIP3D~\cite{zhang2025tapip3d} tracks in camera-stabilized feature clouds, SpatialTrackerV2~\cite{xiao2025spatialtrackerv2} jointly regresses geometry, ego-motion, and point motion end-to-end, and geometry engines such as VGGT~\cite{wang2025vggt} and MegaSaM~\cite{li2025megasam} recover camera poses and dense depth in one pass. Against this backdrop, the proposed method differs in five ways. \emph{(i)~Representation:} to our knowledge it is the first to represent the tracked target as a \emph{fixed} affine (barycentric) combination of \emph{observed}, semantically selected anchor points, together with a proof that the weights are exactly invariant to any rigid change of coordinates (Property~\ref{prop:lincomb}); this turns the virtual-control-point idea of EPnP~\cite{lepetit2009epnp} inside out---the control points are measured, and it is the weights that persist through time as the object identity of the target. \emph{(ii)~Output:} it produces the full 6-DoF pose---rotation \emph{and} translation---of an arbitrary target point, even one that is occluded or textureless, whereas 3D point trackers output trajectories only, and camera-geometry engines output only the global camera pose rather than a local scene-anchored frame for content placement. \emph{(iii)~Simplicity:} it is training-free and closed-form (constrained least squares plus one SVD per frame), a thin geometric layer over any off-the-shelf 2D tracker and depth network, and therefore inherits---at zero cost---every future improvement of those components. \emph{(iv)~Semantics used constructively:} whereas dynamic-SLAM pipelines use segmentation to \emph{discard} unreliable regions~\cite{bojko2022selfimproving}, we use it to \emph{select} world-static anchors near the target, which simultaneously stabilizes tracking and conditions the depth noise (Remark~\ref{rem:anchors}). In addition, the explicit weight vector doubles as an interpretable, per-anchor diagnostic: anchors can be dropped or re-weighted on the fly (steps S3--S4 of Algorithm~\ref{alg:main}) without retraining anything. \emph{(v)~A spectral bridge to structure from motion:} applying the representation to the anchors themselves yields the fixed-point equation $\mathbf{P} = \mathbf{W}\mathbf{P}$ of Section~\ref{sec:lgae}, whose $\lambda = 1$ eigenspace \emph{is} the scene geometry. While the ingredients have classical echoes---sum-to-one reconstruction weights in LLE~\cite{roweis2000lle}, zero-diagonal self-expression in subspace clustering~\cite{elhamifar2013ssc}, and orthographic factorization in Tomasi--Kanade~\cite{tomasi1992shape}---to our knowledge the combination is new: a view-invariant $K \times K$ auto-regression matrix, estimable from 2D projections alone, whose leading eigenvectors directly reconstruct the 3D structure.

\section{Related work}

\textbf{Point tracking.} TAPIR~\cite{doersch2023tapir} and CoTracker~\cite{karaev2024cotracker} track arbitrary 2D points across long videos, the latter jointly over many points, with robustness to short occlusions. We use such a tracker as a black box for the anchors, and our method extends its output from 2D position to full 3D pose.

\textbf{3D point tracking.} A recent line of work extends point tracking to 3D. SpatialTracker~\cite{xiao2024spatialtracker} tracks 2D pixels in a tri-plane 3D space; DELTA~\cite{ngo2025delta} tracks densely and efficiently over long ranges; TAPIP3D~\cite{zhang2025tapip3d} tracks in camera-stabilized spatio-temporal feature clouds; SpatialTrackerV2~\cite{xiao2025spatialtrackerv2} unifies scene geometry, camera ego-motion, and point-wise motion in a single differentiable pipeline. More recently, TrackingWorld~\cite{lu2025trackingworld} tracks almost all pixels in world coordinates, PointSt3R~\cite{guerrier2025pointst3r} repurposes a 3D reconstruction backbone for tracking, and MVTracker~\cite{rajic2025mvtracker} tracks points from several calibrated views at once. Progress is measured on the TAPVid-3D benchmark~\cite{koppula2024tapvid3d}. These methods output metric 3D \emph{trajectories} of query points; they do not directly provide the rotation of a local, scene-anchored frame around an arbitrary target, which is precisely what content placement requires and what our anchor construction supplies in closed form.

\textbf{Feedforward video geometry.} VGGT~\cite{wang2025vggt} regresses camera poses, dense depth, and point maps from an image set in one transformer pass, MegaSaM~\cite{li2025megasam} recovers accurate camera parameters and depth from casual dynamic videos, and any-view models such as $\pi^3$~\cite{wang2026pi3} and Depth Anything~3~\cite{lin2025depthanything3} predict consistent geometry from an arbitrary number of views, with or without known poses. Such engines could replace or complement our per-frame depth and intrinsics inputs; they estimate \emph{camera} pose, whereas we estimate the pose of a designated \emph{scene point} in the camera frame---the two are complementary.

\textbf{Monocular depth.} Depth foundation models such as Depth Anything~V2~\cite{yang2024depthanything} produce dense (metric) depth from a single frame, but per-pixel estimates remain noisy and temporally inconsistent, which motivates our redundant multi-anchor, multi-frame formulation.

\textbf{Camera pose from control points.} EPnP~\cite{lepetit2009epnp} expresses $n$ 3D points as weighted sums of four virtual control points with weights that sum to one, in order to estimate the camera pose. 

We take the opposite direction: we do not compute the camera pose, but are only interested in the linear relationship between potentially many anchor points (not just four), in order to robustly and rapidly estimate their 3D position from noisy monocular depth estimates---which has become possible only in recent years, with modern deep nets. Note that the linear relationship in our case between any given point and many others is estimated once, using linear least squares over points matched across multiple frames, and the same weights are then used in any future frame. The need for a larger number of anchor points is justified by the noisy estimation of monocular depth, per frame. Then, at test time and from any unknown camera pose, the 3D location of a given point is estimated immediately from the linear combination of the noisy depth estimates of the others---thus, each point plays the role of an anchor in the estimation of the others. The expectation is that the linear regression from multiple anchor points, each lifted with monocular depth, improves the 3D location of a given point, over its initial single-point monocular depth estimation.

\textbf{Barycentric self-expression and spectral geometry.} The auto-regression of Section~\ref{sec:lgae} has three classical relatives. Locally linear embedding~\cite{roweis2000lle} reconstructs each data point from its \emph{neighbors} with sum-to-one weights---invariant, like ours, to rotations, rescalings, and translations---and recovers a low-dimensional embedding from the bottom eigenvectors of $(\mathbf{I}-\mathbf{W})^{\!\top}(\mathbf{I}-\mathbf{W})$; we use \emph{global} barycentric weights, and the coordinates appear directly in the $\lambda = 1$ eigenspace of $\mathbf{W}$ itself. Self-expressive models in subspace clustering~\cite{elhamifar2013ssc} solve $\mathbf{X} = \mathbf{X}\mathbf{C}$ with $\operatorname{diag}(\mathbf{C}) = \mathbf{0}$ to build an affinity for clustering; we impose the same masked self-expression, but with the affine constraint, and exploit its \emph{eigenstructure} for reconstruction rather than its sparsity pattern for grouping. Finally, Tomasi--Kanade factorization~\cite{tomasi1992shape} recovers shape and motion from the rank-3 structure of the $2F \times K$ measurement matrix under orthography---a line recently revived with deep canonicalization and sequence modeling for the non-rigid case~\cite{deng2025nrsfm}; our formulation compresses the same multi-view information into a view-\emph{invariant} $K \times K$ weight matrix and reads the structure off its spectrum, inheriting the affine ambiguity familiar from affine structure from motion~\cite{koenderink1991affine,weinshall1993invariants,weinshall1995linear}. That affine combinations remain valid in 2D orthographic views also echoes the classical result that novel orthographic views of an object are linear combinations of a few model views~\cite{ullman1991linear}. Laplacian-style mesh smoothing~\cite{taubin1995signal} similarly iterates local averaging operators on vertex coordinates, and graph-Laplacian regularization remains a strong model-based tool for point-cloud denoising~\cite{zeng2019glr}; our power-iteration auto-correction (Section~\ref{sec:lgae-apps}) replaces generic smoothness priors with weights that exactly encode the scene's own geometry. The masking of each point in its own prediction is, in spirit, a geometric analogue of masked autoencoding~\cite{he2022mae}, applied directly to point sets in Point-MAE~\cite{pang2022pointmae}---though there the predictor is a learned transformer, whereas the LGAE of Section~\ref{sec:lgae} is closed-form and linear.

\textbf{Semantics for static structure.} Dynamic-SLAM systems use semantic segmentation to discard features on moving objects and keep only static scene structure~\cite{bojko2022selfimproving}. We apply the same reasoning constructively: anchors are placed on buildings, trees, and other man-made or static structures near the target point, as identified by a segmentation network~\cite{kirillov2023sam}.

\section{Approach}

\subsection{Setup and notation}

The $K$ anchor points are tracked in every frame (time step) $i$ using a standard single-object tracking algorithm. Let $\bp^{(i)}_1, \bp^{(i)}_2, \dots, \bp^{(i)}_K$ denote the 2D positions of anchors $1,\dots,K$ in frame $i$. Their depths $d^{(i)}_1, d^{(i)}_2, \dots, d^{(i)}_K$ in every frame are computed by a standard neural monocular depth model. Given the camera intrinsics $\bK$, we estimate each anchor's 3D location with respect to the camera of frame $i$ by back-projection:
\begin{equation}
\label{eq:lift}
\bP^{(i)}_k \;=\; d^{(i)}_k \, \bK^{-1}
\begin{bmatrix} \bp^{(i)}_k \\ 1 \end{bmatrix},
\qquad k = 1, \dots, K ,
\end{equation}
where $\bigl[\bp^{(i)\,\top}_k,\, 1\bigr]^{\!\top}$ is the anchor's 2D position in the image, written in homogeneous coordinates.

\subsection{The key property: frame-invariant linear combinations}
\label{sec:property}

There are two key ideas. The first: any point $\bP$ in the scene, fixed with respect to any coordinate system (or the camera), can be written---in homogeneous coordinates---as a linear combination of $4$ or more anchor points. Let us state this precisely:

\begin{property}
\label{prop:lincomb}
Any point $\bP$ in the scene that is fixed with respect to the world coordinate system can be written as a \underline{fixed} linear combination of $K \ge 4$ anchor points (also fixed with respect to the world),
\begin{equation}
\label{eq:comb}
\bP^{(i)} \;=\; w_1 \bP^{(i)}_1 + w_2 \bP^{(i)}_2 + \dots + w_K \bP^{(i)}_K ,
\end{equation}
with \underline{fixed} weights $w_1, w_2, \dots, w_K$, regardless of the coordinate system $i$ (i.e.\ the camera or frame $i$) in which the points are expressed---as long as all points are written in homogeneous coordinates:
\begin{equation}
\label{eq:homog}
\bP^{(i)} = \begin{bmatrix} x^{(i)} \\ y^{(i)} \\ z^{(i)} \\ 1 \end{bmatrix},
\qquad
\bP^{(i)}_k = \begin{bmatrix} x^{(i)}_k \\ y^{(i)}_k \\ z^{(i)}_k \\ 1 \end{bmatrix}.
\end{equation}
\end{property}

\begin{proof}
Suppose the combination holds in one coordinate system:
\begin{equation}
\bP = w_1 \bP_1 + w_2 \bP_2 + \dots + w_K \bP_K .
\end{equation}
We must show that it remains valid under any rigid 3D transformation $(\bR, \bT)$, which maps
\begin{equation}
\bP' = \bR \bP + \bT, \qquad \bP'_k = \bR \bP_k + \bT .
\end{equation}
Under $(\bR,\bT)$, the right-hand side of the combination becomes
\begin{align}
\sum_{k=1}^{K} w_k \left( \bR \bP_k + \bT \right)
&= \bR \underbrace{\left( \sum_{k=1}^{K} w_k \bP_k \right)}_{=\,\bP}
 \;+\; \bT \underbrace{\left( \sum_{k=1}^{K} w_k \right)}_{=\,1}
\label{eq:star} \\
&= \bR \bP + \bT \;=\; \bP' . \nonumber
\end{align}
The crucial step is $\sum_k w_k = 1$, which holds automatically because the combination is written in homogeneous coordinates: the fourth row of
\begin{equation}
\begin{bmatrix} \bP \\ 1 \end{bmatrix}
= w_1 \begin{bmatrix} \bP_1 \\ 1 \end{bmatrix}
+ w_2 \begin{bmatrix} \bP_2 \\ 1 \end{bmatrix}
+ \dots
+ w_K \begin{bmatrix} \bP_K \\ 1 \end{bmatrix}
\end{equation}
reads exactly $1 = w_1 + w_2 + \dots + w_K$. Hence the same weights reproduce $\bP'$ after the transformation.
\end{proof}

We have proved that if the points are written in homogeneous coordinates and the linear combination \eqref{eq:comb} is true in one coordinate system, it is also true after any rotation and translation, i.e.\ in any other coordinate system. Note that the result is \emph{not} trivial: it requires $\sum_k w_k = 1$, otherwise it fails under arbitrary translations.

\subsection{Estimating the weights}

\begin{observation}[Least-squares estimation]
\label{obs:ls}
The linear combination can be estimated with least squares as long as we have the estimated 3D anchor points $\bP^{(i)}_1, \dots, \bP^{(i)}_K$ (the anchors) in at least $N$ frames such that
\begin{equation}
\label{eq:counting}
\underbrace{3 \times N}_{\text{3 eq. per frame}} \;+ \underbrace{\;1\;}_{\substack{\text{constraint}\\ \sum_k w_k = 1}} \;\ge\; \underbrace{\;K\;}_{\text{number of unknowns}} .
\end{equation}
For $N = 1$ (one frame, one camera) and $K = 4$ (four anchor points) we have a unique solution in the general case.
\end{observation}

\begin{observation}[Why many anchors and many frames]
\label{obs:robust}
The problem with estimating a fixed set of weights from a single frame is that the given 3D points $\bP_k$ (the lifted anchors) are not accurate. That is why we consider a relatively large $K$ ($K \gg 4$) and estimate the weights from several frames, obtaining a more robust solution from an overconstrained system of equations solved by least squares. Concretely, stacking \eqref{eq:comb} over frames $i = 1,\dots,N$ yields
\begin{equation}
\label{eq:lsq}
\hat{\bw} \;=\; \arg\min_{\bw \in \R^K} \;
\sum_{i=1}^{N} \Big\| \bP^{(i)} - \sum_{k=1}^{K} w_k \bP^{(i)}_k \Big\|^2
\quad \text{s.t.} \quad \sum_{k=1}^{K} w_k = 1 .
\end{equation}
\end{observation}

\begin{observation}[Accuracy grows with $K$]
When $K$ is large, the solution for any target $\bP$ is expected to be more accurate, since each individual anchor $\bP_k$ is noisy and the redundancy averages this noise out.
\end{observation}

\subsection{Anchor 3D: Recovering 3D pose from a set of anchors}

\begin{observation}[Position and rotation in every frame]
\label{obs:pose}
Having tracked the anchor points in image $i$, we can find $\bP$ as their linear combination with the estimated weights. More importantly, we can also compute the correct \emph{rotation} by the same rapid approach since the 3D axes of the coordinate system are vectors (differences between two 3D points) that can be expressed in the exact same linear fashion as single points. Of course, a classical approach could also be used, such as the
SVD (Kabsch--Umeyama alignment~\cite{umeyama1991least}) applied to the anchor point sets in the reference image $1$ and the current image $i$, yielding $(\bR^{(i)}, \bT^{(i)})$. Thus we have a complete estimation of the 3D pose and location of $\bP$ in image $i$ with respect to camera $i$:

\begin{equation}
\label{eq:pose}
\bP^{(i)} \;=\; \bR^{(i)} \bP + \bT^{(i)}
\qquad \text{(from the estimated $\bR$ and $\bT$, for any point in the scene)},
\end{equation}
and equivalently, through the anchors,
\begin{equation}
\label{eq:combi}
\bP^{(i)} \;=\; \sum_{k=1}^{K} w_k \, \bP^{(i)}_k ,
\end{equation}
where $\bP^{(i)}_k$ are the estimated 3D anchor points in image $i$ (with respect to camera $i$).
\end{observation}

\begin{observation}[Robustness to occlusion]
The main advantage of tracking a point's location through $K$ anchors, instead of tracking the point directly with a standard tracker, is that we can (very often) rely on stable and visible anchors even when the point in question is occluded, or when it sits in an area that is hard to track---a region with zero image gradients or constant color.
\end{observation}

\begin{observation}[Semantic anchor selection]
We use semantic segmentation to choose the anchor points wisely; see Remark~\ref{rem:anchors}.
\end{observation}

%\subsection{Algorithm}

\begin{algorithm}[ht]
\caption{Anchor 3D: Pose Tracking from Noisy Depth with Semantic Segmentation}
\label{alg:main}
\begin{algorithmic}[1]
\State \textbf{S1:} Estimate $w_1, w_2, \dots, w_K$ with least squares \eqref{eq:lsq} from the first $N$ frames, for a given point $\bP$ and anchors $\bP_1, \dots, \bP_K$, by using a standard tracker and per-frame depth estimation \eqref{eq:lift} to form the overconstrained system of equations.
\State \textbf{S2:} Use the estimated weights and the tracked anchors to find $\bP$ in subsequent frames via \eqref{eq:combi}, and the rotation via the same fast linear estimation (our approach) or the classic
SVD alignment (Observation~\ref{obs:pose}).
\State \textbf{S3:} If not all anchors are found in the current image $i$, recompute the weights for the \emph{available} anchors using the initial $N$ frames, as long as $K_{\mathrm{available}} \ge K_{\min}$.
\State \textbf{S4:} If $K_{\mathrm{available}} < K_{\min}$, restart the estimation of the weights by going back to \textbf{S1} (re-track all points with a standard tracker).
\end{algorithmic}
\end{algorithm}

\section{Discussion}

\begin{remark}[Choosing the anchors]
\label{rem:anchors}
Anchors are chosen in regions that are expected to be fixed with respect to the world---buildings, trees, man-made structures---by using semantic segmentation, and \underline{near} the point to be tracked $\bP$. Closeness to $\bP$ is important in order to reduce numerical estimation errors: depth estimation around $\bP$ is expected to be \ul{more consistent with the depth of $\bP$} itself.
\end{remark}

\begin{remark}[When to use 3D pose tracking vs.\ standard tracking]
\label{rem:ensemble}
The two tracking approaches can be combined (an average ensemble). When standard tracking fails---signaled by a large disagreement between the two---the 3D anchor-based tracking is the more reliable of the two.
\end{remark}

\begin{remark}[Rotation requires the anchors]
\label{rem:avatar}
For correctly rotating an avatar with respect to the world, the anchor points are always needed: standard tracking of $\bP$ provides no rotation information, only position (translation).
\end{remark}

\section{Anchor 3D algorithm as geometric auto-regression}
\label{sec:lgae}

We can use the idea of representing the 3D location of a point as a linear combination of $K$ anchor points in order to represent each \emph{anchor} as a linear combination of the other $K-1$ anchors---so that the anchors can auto-correct themselves, as a group, through auto-regression. This will lead to a surprising result.

\subsection{The auto-regression matrix}

First, let us write, for an anchor $\bP_i$,
\begin{equation}
\label{eq:autoreg}
\bP_i \;=\; w_{i1}\bP_1 + w_{i2}\bP_2 + \dots + 0\cdot\bP_i + \dots + w_{iK}\bP_K ,
\end{equation}
i.e.\ the same equation as \eqref{eq:comb}, written for each of the $K$ anchor points in turn, with the convention $w_{ii} = 0$: anchor $\bP_i$ never participates in predicting itself. The weights $w_{ij}$ for writing any anchor as a linear combination of the others can be estimated exactly as before (Observation~\ref{obs:ls}). Collecting the weight $w_{iq}$ of anchor $q$ in the equation of anchor $i$, we form the matrix
\begin{equation}
\label{eq:Wmatrix}
\mathbf{W} \;=\;
\begin{bmatrix}
0 & w_{12} & w_{13} & \cdots & w_{1K} \\
w_{21} & 0 & w_{23} & \cdots & w_{2K} \\
\vdots & & \ddots & & \vdots \\
w_{K1} & w_{K2} & \cdots & w_{K,K-1} & 0
\end{bmatrix},
\qquad w_{ii} = 0, \quad \textstyle\sum_{q} w_{iq} = 1 .
\end{equation}
Thus $\mathbf{W}$ is a matrix of auto-regressing weights where we always have \underline{$w_{ii} = 0$}, because anchor point $\bP_i$ never participates in predicting itself---we have a \ul{masking of $\bP_i$ for producing $\bP_i$}. Now all $K$ equations can be written at once in matrix form. Stacking the anchor coordinates as the rows of $\mathbf{P} \in \R^{K\times 3}$:
\begin{equation}
\label{eq:stacked}
\underbrace{\begin{bmatrix}
P_{1x} & P_{1y} & P_{1z} \\
P_{2x} & P_{2y} & P_{2z} \\
\vdots & \vdots & \vdots \\
P_{ix} & P_{iy} & P_{iz} \\
\vdots & \vdots & \vdots \\
P_{Kx} & P_{Ky} & P_{Kz}
\end{bmatrix}}_{\textstyle \mathbf{P}}
\;=\;
\underbrace{\begin{bmatrix}
0 & w_{12} & w_{13} & \cdots & w_{1K} \\
w_{21} & 0 & w_{23} & \cdots & w_{2K} \\
\vdots & & \ddots & & \vdots \\
w_{i1} & \cdots & w_{ii}\!=\!0 & \cdots & w_{iK} \\
\vdots & & & \ddots & \vdots \\
w_{K1} & w_{K2} & \cdots & w_{K,K-1} & 0
\end{bmatrix}}_{\textstyle \mathbf{W}}
\cdot
\underbrace{\begin{bmatrix}
P_{1x} & P_{1y} & P_{1z} \\
P_{2x} & P_{2y} & P_{2z} \\
\vdots & \vdots & \vdots \\
P_{ix} & P_{iy} & P_{iz} \\
\vdots & \vdots & \vdots \\
P_{Kx} & P_{Ky} & P_{Kz}
\end{bmatrix}}_{\textstyle \mathbf{P}} ,
\end{equation}
in short, the matrix fixed-point equation
\begin{equation}
\label{eq:fixedpoint}
\boxed{\;\mathbf{P} \;=\; \mathbf{W}\,\mathbf{P}\;}
\end{equation}

\begin{observation}[A linear geometric auto-encoder]
The diagonal of zeros of $\mathbf{W}$ means that each $\bP_i = [P_{ix},P_{iy},P_{iz}]$ is actually masked when predicting its location $\bP_i$ from the others. So the equation is, in essence, a \ul{linear geometric auto-encoder} (LGAE): the scene encodes---and reproduces---itself, each point being explained by all the others.
\end{observation}

\begin{observation}[The geometry lives in the $\lambda = 1$ eigenspace]
\label{obs:eigen}
Reading \eqref{eq:fixedpoint} column by column,
\begin{equation}
\label{eq:eigencols}
\mathbf{p}_x = \mathbf{W}\mathbf{p}_x, \qquad
\mathbf{p}_y = \mathbf{W}\mathbf{p}_y, \qquad
\mathbf{p}_z = \mathbf{W}\mathbf{p}_z ,
\end{equation}
where the columns of $\mathbf{P}$ are
\begin{equation}
\label{eq:columns}
\mathbf{p}_x = \begin{bmatrix} P_{1x} \\ P_{2x} \\ \vdots \\ P_{ix} \\ \vdots \\ P_{Kx} \end{bmatrix},
\qquad
\mathbf{p}_y = \begin{bmatrix} P_{1y} \\ P_{2y} \\ \vdots \\ P_{iy} \\ \vdots \\ P_{Ky} \end{bmatrix},
\qquad
\mathbf{p}_z = \begin{bmatrix} P_{1z} \\ P_{2z} \\ \vdots \\ P_{iz} \\ \vdots \\ P_{Kz} \end{bmatrix}.
\end{equation}
In other words, the first three eigenvectors of $\mathbf{W}$ have $\lambda = 1$ and encode the 3D positions of the anchors ($\lambda = 1$ with multiplicity at least $3$). Moreover, since each row sums to one, the constant vector $\mathbf{1}$ is a fourth $\lambda = 1$ eigenvector: $\mathbf{W}\mathbf{1} = \mathbf{1}$ (this is the homogeneous coordinate). Any linear combination $\alpha\,\mathbf{p}_x + \beta\,\mathbf{p}_y + \gamma\,\mathbf{p}_z + \delta\,\mathbf{1}$ is again an eigenvector with $\lambda = 1$, so the eigenspace contains the full affine span of the coordinate functions.
\end{observation}

\begin{observation}[Spectral coordinate-frame invariance]
Observation~\ref{obs:eigen} confirms, from the spectral side, that the actual world coordinate system in which $\mathbf{p}_x, \mathbf{p}_y, \mathbf{p}_z$ are written is not important: rotating and translating the points arbitrarily replaces each coordinate vector by a linear combination of $\{\mathbf{p}_x, \mathbf{p}_y, \mathbf{p}_z, \mathbf{1}\}$, which stays inside the $\lambda = 1$ eigenspace---and $\mathbf{W}$ itself remains \emph{unchanged} (Property~\ref{prop:lincomb}). The matrix $\mathbf{W}$ is thus a coordinate-free, view-invariant encoding of the scene's shape.
\end{observation}

\section{Camera-independent spectral 3D reconstruction from geometric auto-regression}
\label{sec:lgae-apps}

We make the following reasoning steps, based on known mathematical results~\cite{hormann2017gbc,horn2013matrix}, to show that an effective and very fast power iteration algorithm can recover the 3D shape of the points within a common coordinate system.

\begin{enumerate}

\item Any point in the interior of the convex hull of a set of points can be expressed as a convex combination of those points, i.e.\ an affine combination with non-negative weights whose sum is equal to one (generalized barycentric coordinates~\cite{hormann2017gbc}).

\item A subset of our points lies exactly on the boundary of the convex hull, including the points that are its vertices. Therefore, for the majority of the points, the corresponding weights $\bw$ can be chosen non-negative, the potentially problematic cases being the points located at the vertices of the convex hull.

\item Thus, the only remaining points for which negative weights may potentially occur are the vertices of the convex hull. For these points, instead of estimating the corresponding weights, we keep their positions exactly as given by the monocular depth estimate. This is achieved by setting the corresponding diagonal entry of the matrix $\mathbf{W}$ to one, with all other entries on that row set to zero. Consequently, multiplication by $\mathbf{W}$ leaves these points unchanged, through the identity mapping. In this way, we guarantee that all entries of the auto-regression matrix $\mathbf{W}$ are non-negative.

\item It follows that $\mathbf{W}$ is a \textbf{row-stochastic (right-stochastic) matrix}, since all its entries are non-negative and the entries of every row sum to one.

\item Consequently, by the Perron--Frobenius theorem~\cite{horn2013matrix}, the spectral radius of $\mathbf{W}$ is equal to one; equivalently, the maximum absolute value of its eigenvalues is
\begin{equation*}
\rho(\mathbf{W})=\max_i |\lambda_i|=1.
\end{equation*}

\item Given that the eigenvalues of maximum modulus are equal to one, and assuming that the eigenvalue $\lambda=1$ has multiplicity four, the corresponding principal eigenspace can be identified with the four-dimensional space spanned by the homogeneous 3D coordinates of the points, expressed with respect to an arbitrary coordinate system and up to the inherent geometric ambiguities.

\item Therefore, power iteration applied to $\mathbf{W}$, starting from a random vector, converges to a vector in the dominant eigenspace associated with $\lambda=1$. By repeatedly extracting independent directions from this four-dimensional eigenspace, we can recover a basis spanning the four homogeneous coordinate dimensions of the 3D point configuration.

More specifically, after recovering an eigenvector $\mathbf{v}$, its contribution can be removed through an appropriate deflation or orthogonalization procedure~\cite{golub2013matrix}, after which power iteration can be applied again to recover another independent direction. Repeating this procedure four times yields a basis for the four-dimensional dominant eigenspace and, consequently, a homogeneous representation of the 3D point configuration, up to the corresponding global geometric ambiguity.

In essence, this procedure provides a robust reconstruction of the 3D point configuration by exploiting the auto-regressive relationships encoded by $\mathbf{W}$, potentially yielding a more stable estimate than the 3D geometry obtained directly from monocular depth estimation applied independently to a single frame.
\end{enumerate}

\subsection{Spectral 3D reconstruction using multiple monocular depth maps}
\label{sec:Spec_3D_1}

The reasoning steps above suggest the following reconstruction procedure. First, construct a non-negative auto-regression matrix $\mathbf{W}$ whose rows sum to one, using convex reconstruction weights whenever possible and identity rows for points that cannot be reconstructed from the remaining points. Then compute the four-dimensional invariant subspace associated with the eigenvalue $1$, or equivalently the four right singular vectors of $\mathbf{I}-\mathbf{W}$ with the smallest singular values.

Since the constant vector ($\mathbf{1}$) is known a priori, only three additional independent directions need to be estimated. These directions provide a 3D embedding that is consistent with all auto-regressive constraints encoded by $\mathbf{W}$.

In contrast to monocular depth estimation performed independently for a single frame, the resulting reconstruction is constrained globally by the geometric relationships among all points. Consequently, errors affecting individual depth estimates cannot independently perturb the reconstructed geometry; instead, the final solution is restricted to the low-dimensional invariant subspace jointly supported by the entire point configuration.

\subsection{3D auto-correction by power iteration}
As discussed, the fixed-point equation $\mathbf{P} = \mathbf{W}\mathbf{P}$ can be used to auto-correct a noisy set of 3D points by power iteration:
\begin{equation}
\label{eq:power}
\mathbf{P}^{(t+1)} \;\leftarrow\; \mathbf{W}\,\mathbf{P}^{(t)} .
\end{equation}
If the anchor points are the vertices of a mesh, the method can be used to \ul{improve depth maps and meshes in 3D}: each vertex is repeatedly re-explained by all the others, and errors inconsistent with the group are averaged away.

\subsection{Spectral 3D reconstruction without monocular depth maps}
\label{sec:Spec_3D_2}

Since the linear combination holds in 4D homogeneous coordinates, it also holds for any \emph{subset} of the coordinates---say $(X, Y)$. This suggests the following idea: take the 2D projections of the $K$ points from many (random) camera poses onto the image plane, assuming the projection is approximately orthographic (we look at the points from some distance). From many such projections, we estimate the auto-regressing weight matrix $\mathbf{W}$ with least squares, using the same algorithm as before---so $\mathbf{W}$ is estimated \emph{from 2D projections only}, without ever needing estimated 3D locations from depth maps.

Once $\mathbf{W}$ is estimated, we use the knowledge from Observation~\ref{obs:eigen} that its $3$ eigenvectors with $\lambda = 1$ that are different from $\mathbf{1}$ are in fact
\begin{equation}
\label{eq:discover}
\mathbf{v}_1 = \mathbf{p}_x, \qquad \mathbf{v}_2 = \mathbf{p}_y, \qquad \mathbf{v}_3 = \mathbf{p}_z .
\end{equation}
By computing its eigenvectors $(\mathbf{v}_1, \mathbf{v}_2, \mathbf{v}_3)$ with $\lambda = 1$, such that $\mathbf{v}_i \neq \mathbf{1}$ (the constant vector of all ones), we \emph{discover} $\mathbf{p}_x, \mathbf{p}_y, \mathbf{p}_z$---the actual (scaled) 3D positions of the points, up to the affine change of basis that mixes $\{\mathbf{p}_x,\mathbf{p}_y,\mathbf{p}_z,\mathbf{1}\}$ within the eigenspace (the classical affine ambiguity of orthographic structure from motion~\cite{koenderink1991affine,weinshall1995linear}, upgradable to metric with standard constraints). 

The algorithm for \textbf{Spectral 3D reconstruction without monocular depth maps} becomes, in a nutshell: 1) estimate $\mathbf{W}$ from many near-orthographic 2D views by least squares; 2) read the 3D structure off the $\lambda = 1$ eigenvectors of $\mathbf{W}$.

\section{Conclusion}

We started by developing a method that turns three noisy but ubiquitous components---2D point tracking, monocular depth, and semantic segmentation---into a robust estimator of full 3D pose in a fixed world frame. Its correctness rests on a simple but non-trivial invariance property of affine combinations in homogeneous coordinates (Property~\ref{prop:lincomb}), and its robustness on redundancy: many semantically chosen anchors, many frames, and least squares.

Turning the representation on the anchors themselves then produced the linear geometric auto-encoder $\mathbf{P} = \mathbf{W}\mathbf{P}$ of Section~\ref{sec:lgae}: a view-invariant weight matrix whose $\lambda = 1$ eigenspace is the scene geometry itself, usable both to denoise 3D points and meshes by power iteration and to reconstruct 3D structure from near-orthographic 2D projections alone.

Section~\ref{sec:lgae-apps} then turned the auto-regression into a reconstruction method: restricting the weights to be non-negative---always possible for points inside the convex hull of the anchors---and keeping the hull vertices fixed through identity rows makes $\mathbf{W}$ row-stochastic, so its spectral radius is one and its dominant eigenspace is spanned, under the stated multiplicity assumption, by the four homogeneous coordinate vectors $\{\mathbf{p}_x, \mathbf{p}_y, \mathbf{p}_z, \mathbf{1}\}$. 

The resulting algorithm is simple: estimate $\mathbf{W}$ by least squares, then recover the 3D configuration of the points in a common coordinate system by power iteration with deflation, or equivalently from the four smallest singular vectors of $\mathbf{I} - \mathbf{W}$, up to a global affine ambiguity and without estimating any camera pose. 

The last two algorithms follow from the two ways of obtaining $\mathbf{W}$: estimated from noisy monocular depth over several frames, it auto-corrects depth-lifted points and meshes by power iteration; estimated from near-orthographic 2D projections alone, it reconstructs 3D structure from 2D correspondences without any depth input. In both cases every point is constrained jointly by all the others, rather than by an independent per-frame depth estimate.

Natural next steps are an experimental validation on standard tracking benchmarks, an analysis of the sensitivity of the weights to depth noise as a function of $K$ and $N$, the ensemble combination with standard trackers of Remark~\ref{rem:ensemble}, and a study of the spectral reconstruction of Section~\ref{sec:lgae-apps} against factorization-based structure from motion~\cite{tomasi1992shape} in noise, perspective, and missing-data regimes.

\section{Final Remarks}

Please note that none of our derivations and algorithms use the camera pose or the camera projection matrix. The only camera-related quantity that appears anywhere is the intrinsic matrix $\mathbf{K}$, used in \eqref{eq:lift} to back-project noisy monocular depth into 3D, and even this is not needed by the last method. Our methods rest solely on the property that 3D points can be written as stable linear combinations of each other, regardless of the camera pose (Property~\ref{prop:lincomb}). This is in striking contrast with existing 3D geometry reconstruction methods, which estimate the camera motion, explicitly or implicitly, as part of the solution. We rely only on noisy monocular depth estimation, which has become available only in recent years, while the very last method---presented in Section~\ref{sec:Spec_3D_2}---makes the assumption that the linear combination between 2D points in the image is the same as the one between their 3D counterparts. This is true under orthographic or weak perspective projection and, again very importantly, it holds without having to explicitly know or even consider the camera viewpoint and projection matrix. All the calculations, methods, and results are independent of the camera; our auto-regression matrix $\mathbf{W}$ depends only on the coordinate system created implicitly by the 3D points themselves.

Thus $\mathbf{W}$ is the same regardless of the camera viewpoint and regardless of the chosen world coordinate system.

\section{Acknowledgements}

The author would like to thank Dragos Costea and Emanuela Haller for 
helpful discussions and valuable insights into this work.
The work is supported by projects
``Romanian Hub for Artificial Intelligence -- HRIA'',
Smart Growth, Digitization and Financial Instruments
Program, 2021-2027 (MySMIS No. 334906), 
and ``European Lighthouse of AI for Sustainability -- ELIAS'', Horizon Europe program (Grant No.
101120237).

\bibliographystyle{unsrt}
\bibliography{refs}

\end{document}